\documentclass[12pt]{amsart} 
\usepackage{graphicx}
\usepackage{subcaption}
\usepackage[english]{babel}
\usepackage{amsmath, amssymb, amsfonts, amsbsy, amsthm, latexsym, stmaryrd, mathtools, fullpage, enumerate, ytableau, shuffle, multirow, mathdots, enumerate, tikz, verbatim, comment, hhline, hyperref, arydshln, fancyvrb, subcaption}

\usepackage[colorinlistoftodos]{todonotes}

\newtheorem{theorem}{Theorem}[section]

\newtheorem{proposition}[theorem]{Proposition}

\newtheorem{definition}[theorem]{Definition}

\title{Exploration of Numerical Precision in Deep Neural Networks}
\author{Zhaoqi Li, Yu Ma, Catalina Vajiac, Yunkai Zhang}
\date{\today}

\begin{document}

\begin{abstract}
Reduced numerical precision is a common technique to reduce computational cost in many Deep Neural Networks (DNNs). While it has been observed that DNNs are resilient to small errors and noise, no general result exists that is capable of predicting a given DNN system architecture's sensitivity to reduced precision.  In this project, we emulate arbitrary bit-width using a specified floating-point representation with a truncation method, which is applied to the neural network after each batch. We explore the impact of several model parameters on the network's training accuracy and show results on the MNIST dataset. We then present a preliminary theoretical investigation of the error scaling in both forward and backward propagations. We end with a discussion of the implications of these results as well as the potential for generalization to other network architectures. 
\end{abstract}

\maketitle
\section{Introduction}

Despite the advances in hardware and the usage of GPUs nowadays, training a deep neural network (DNN) is still extremely computationally expensive, sometimes taking up to a few months. Most of the memory occupied by DNN attributes to the weight matrices that encode the information of the network, which is primarily represented in 32 bits (single precision). Intuitively, reducing the precision requirement cuts down the amount of data stored, which in turn shortens runtime for compute-bound devices. Not only does reduced precision increase the capacity of devices, it also speeds up the data transferring process, which is a major factor in distributed algorithms. If successfully applied to DNNs, reduced precision would be proven much use in various areas, such as mobile devices.

Neural networks and machine learning algorithms tend to be resilient to error from reduced precision \cite{Gupta}. Therefore, the plausibility of reduced precision has already been investigated in some neural network architectures on standard machine learning datasets \cite{DBLP}. With stochastic rounding, MNIST and CIFAR-10 datasets can be trained up to state-of-art performance with all parameters truncated to 16 bits \cite{Gupta}. With the use of fixed point and dynamic fixed point formats, the parameters can even be truncated to 10 bits \cite{Courbariaux}. Binarized Neural Networks (BNNs), neural networks with binary weights and activations at runtime, and quantization methods are also studied and can achieve nearly state-of-the-art results \cite{DBLP}. 

However, explanations on resiliency and sensibility of reduced precision were primarily empirical \cite{Courbariaux, Ngiam}, and there are few studies on the topic of numerical stability \cite{Higham, Trefethen}. Furthermore, no previous estimate of precision tolerance has been established, and there is no concrete analysis about what aspects of neural networks are influenced by reduced precisions. 

This paper explores the resiliency of different parameters in a neural network to reduced precisions in terms of testing accuracy. We show that reduced precision is insensitive to numerous parameters, but sometimes harmful to network's architecture. In particular, we show that the test accuracy is sensitive to weight initialization and the number of layers in the convolutional neural network. Caution should be taken in the future when implementing reduced precision without a thorough understanding. Our results are based on the benchmark dataset MNIST. 

\section{Methods}

Numerical precision is defined as the measurement of the accuracy at which quantity is expressed and developed an arbitrary precision. Disregard practical concerns, we investigated low precision on a continuous mechanism, where we truncate a certain number of bits from the mantissa part of data representation. An example is illustrated below.

\begin{figure}[h!]
\centering
\begin{BVerbatim}
   0 0111 1111 1111 1100 1100 1100 1100 110   1.9875
&  1 1111 1111 1111 1110 0000 0000 0000 000   16-bit filter
   ----------------------------------------
   0 0111 1111 1111 1100 0000 0000 0000 000   1.984375
   
\end{BVerbatim}
\caption{Example of filter being applied to 32-bit float}
\label{fig:bitwise}
\end{figure}

We adopted truncation by batch to a standard convolutional neural network framework. It was designed to have two convolutional and pooling layers followed by a densely connected layer, shown in Figure~\ref{fig:framework}. 

\begin{figure}[h!]
\centering
\includegraphics[scale=0.35]{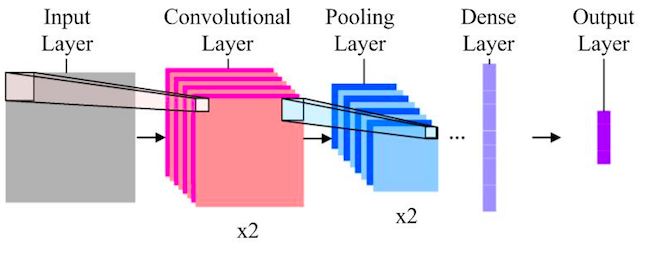}
\caption{A CNN Framework}
\label{fig:framework}
\end{figure}

Weights were first initialized to be small random numbers. ReLu is the activation function and softmax is used to process the output. Training data were propagated in mini-batches of size 128. Within each epoch, we truncated all weights after training. Our analysis of the results is first based on the final accuracy after a certain number of iterations (50 or 100). If it converges to a high level (95\% for MNIST), we then compared how fast it converges. We used Theano for implementing our framework as it is a comparatively open-source library.

We also implemented truncation by layer method which truncated all the weights after going through each layer. We show some results of this method in Section~\ref{Ch:FutureWork}. 

\section{Sensitivity Analysis of Neural Network Parameters}

To test the resilience of CNNs to reduced precision, we chose several parameters which hypothesized to be the major sources of CNN test error. We then perturbed these parameters one by one and analyze their effects under reduced precision. The parameters we investigated include number of layers (convolutional and dense), number of dense units, batch sizes, rounding schemes, and weight initialization conditions. We explored one parameter of interest at a time under different bit sizes. 

\subsection{Bitsize}

We would like to know the smallest bit size needed to converge in the default setting. We trained the CNN for 500 iterations, where an iteration is defined as one propagation of the entire dataset through the network using mini-batch training. The results are shown in Figure \ref{fig:cnn}.

\begin{figure}[ht!]
\includegraphics[width=4in]{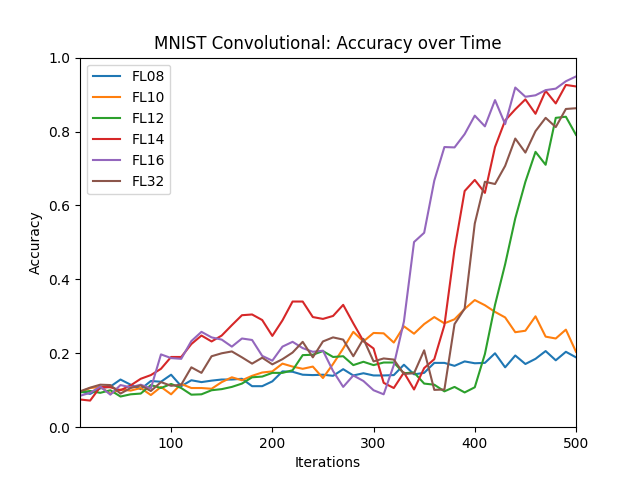}
\centering
\caption{Test accuracy vs. iterations for mini-batch CNN.}
\label{fig:cnn}
\end{figure}

The vertical axis represents the test accuracy and the horizontal axis represents the number of iterations, which we also denote as epochs later in this paper. Note that the accuracy never rised above 40\% when truncating to 8 or 10 bits, but at 12 bits or higher, the accuracy does climb to 80\%. This shows 12 bits as a turning point for MNIST on this specific network.
Unfortunately, we did not find this to be a general result. The network would not converge on CIFAR-10 when all the parameters were truncated to a small bit size as 12 bits. 

\subsection{Rounding Schemes}

To move from high to low precision, a standard routine of handling the excess digits is needed. One basic scheme is truncation, or simply cutting off the mantissa value after a certain number of bits. This will always lead to a smaller-than-original number. An alternative is stochastic rounding, which is a probabilistic rounding method defined as follows: 

\begin{align*}
\textrm{Pr}(x\to\left \lfloor{x}\right \rfloor) &= \left \lceil{x}\right \rceil - x\\
\textrm{Pr}(x\to\left \lceil{x}\right \rceil) &= x - \left \lfloor{x}\right \rfloor,
\end{align*}

where $\left \lfloor{x}\right \rfloor$ is the floor of $x$ and $\left \lceil{x}\right \rceil$ is the ceiling of $x$. In other words, if $x$ is close to $\left \lfloor{x}\right \rfloor$, it has a higher probability of rounding down, but it still has some chance of rounding up. In neural networks, many weights often have around the same value. This practice prevents all of them to be rounded up or down and thus effectively averages out the truncation error. 

We observed that stochastic rounding causes the network to converge at lower bit sizes where truncation fails. While stochastic rounding does not affect the final accuracy for high precision, it does provide faster convergence in many cases, though not all because of the intrinsic degree of randomness in the method. On average, stochastic rounding improved the rate of convergence by 25\%, as is shown in Figure~\ref{fig:stochastic_round}.

\begin{figure}[h!]
    \centering
    \begin{subfigure}[b]{0.4\textwidth}
        \includegraphics[width=0.9\textwidth]{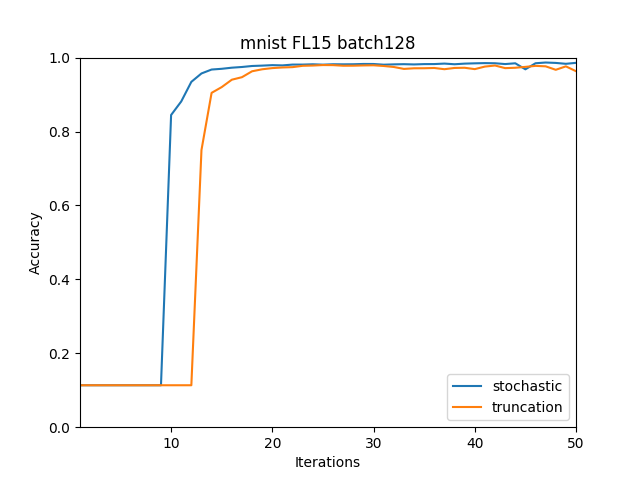}
        \label{fig:15bit}
    \end{subfigure}
    \begin{subfigure}[b]{0.4\textwidth}
        \includegraphics[width=0.9\textwidth]{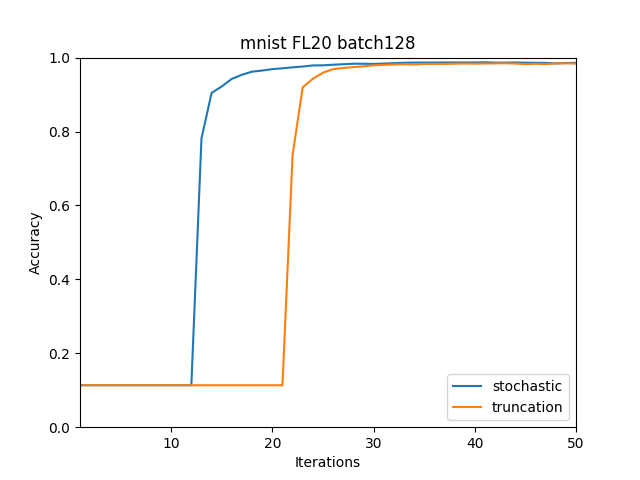}
        \label{fig:20bit}
    \end{subfigure}
    \caption{MNIST Test Accuracy vs. Rounding Scheme}
    \label{fig:stochastic_round}
\end{figure}

\subsection{Number of Dense Layers}\label{sec:layers}

Reduced precision experiments are commonly implemented on structures with many layers, since adding more layers generally improves their performance. However, increasing the number of layers also introduces more rounding errors. Unlike convolution layers, dense layers allow us to control the parameters more precisely. We studied the effects of having from one to five dense layers, each with 100 units.

\begin{figure}[!htb]
    \centering
    \begin{minipage}{.5\textwidth}
        \centering
        \includegraphics[scale=0.4]{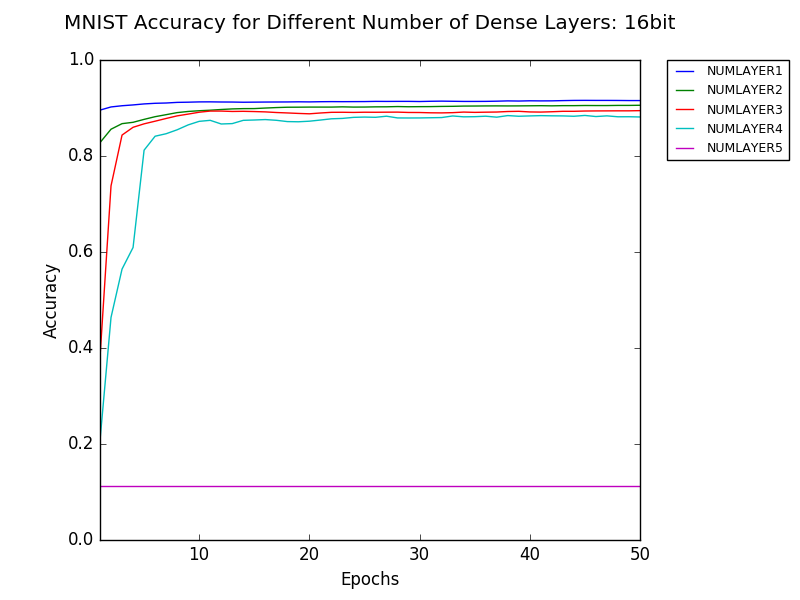}
    \end{minipage}%
    \begin{minipage}{0.5\textwidth}
        \centering
        \includegraphics[scale=0.4]{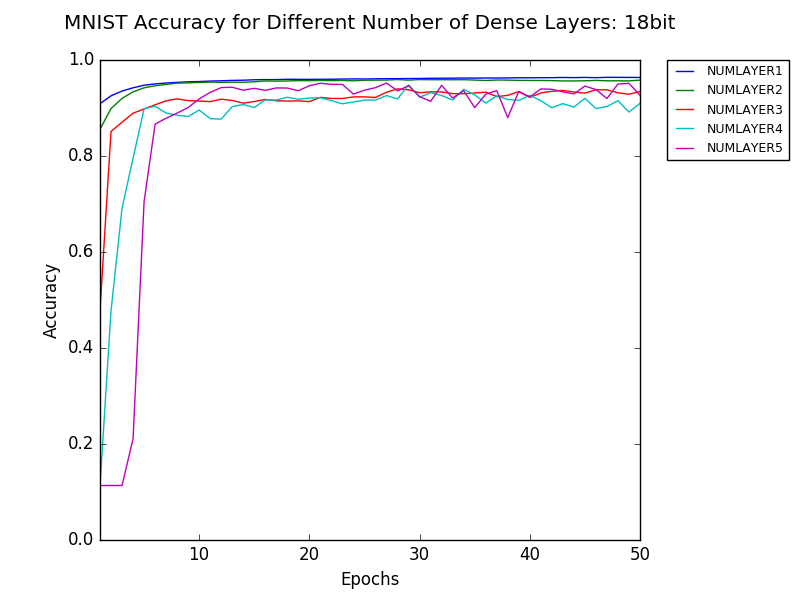}
    \end{minipage}
    \caption{MNIST Test Accuracy vs. Number of Dense Layers}
    \label{fig:acc_dense}
\end{figure}

Result in Figure~\ref{fig:acc_dense} shows that increasing bit size by two could change a network completely from poorly-trained to well-trained. Regardless of the number of layers, the test accuracies increase with the bitsize. However, as the number of layers goes up, the accuracy drops down. In particular, the neural network with five layers does not train when the bitsize is 16, and it fluctuates a lot when the bitsize is 18. Thus, networks with more layers are more sensitive to the number of bits. The reason is that the round-off error tends to accumulate as the number of layers goes up. A detailed error analysis is provided in Section \ref{Ch:analysis}.

\subsection{Number of Dense Units}

The number of dense units represents the number of neurons in a fully connected layer. We tested 160, 130, 110, 100, 90, 70 and 40 units per dense layer (Figure~\ref{fig:dense_unit}). Among all bit sizes, we see that the number of dense units is independent of final accuracy. This implies that a well-trained model may not require the most dense units, which could lead to a more memory-efficient implementation.

\begin{figure}[!htb]
    \centering
    \begin{minipage}{.5\textwidth}
        \centering
        \includegraphics[scale=0.4]{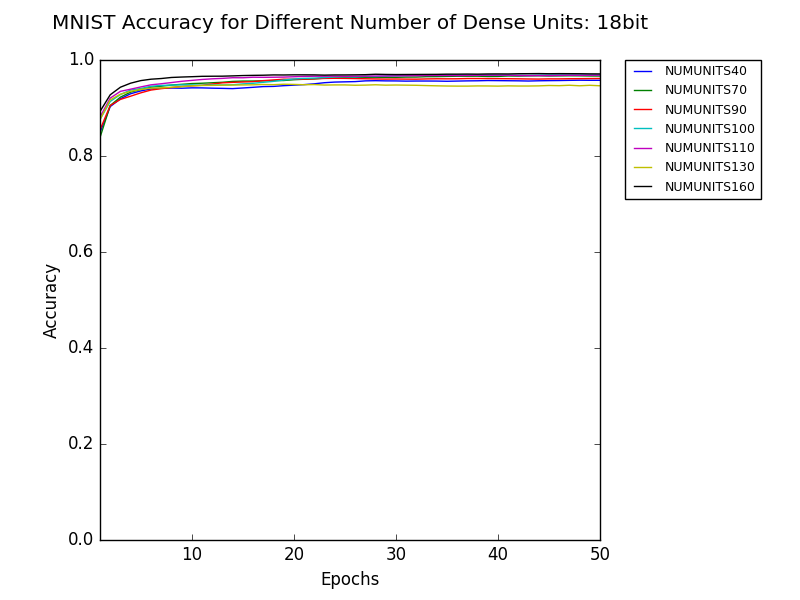}
        \caption{Accuracy vs. Dense Units}
        \label{fig:dense_unit}
    \end{minipage}%
    \begin{minipage}{0.5\textwidth}
        \centering
        \includegraphics[scale=0.4]{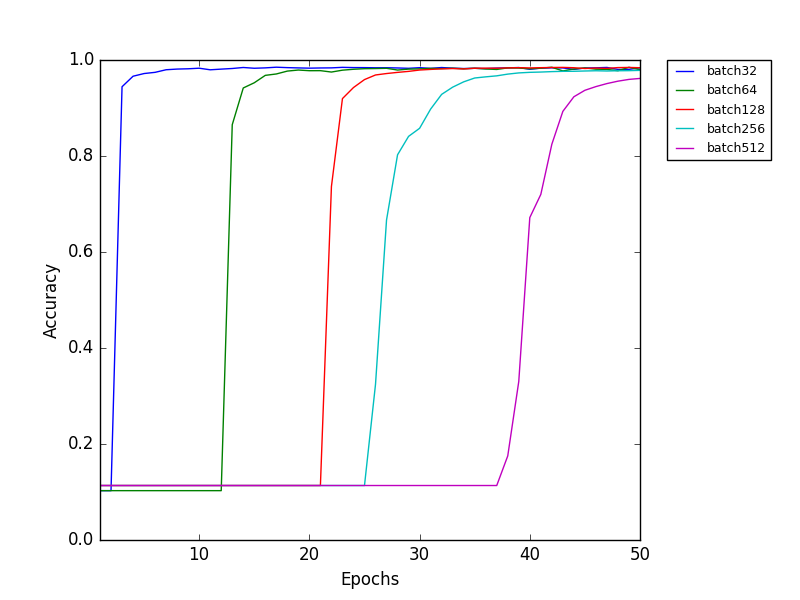}
        \caption{Accuracy vs. Batch Size}
        \label{fig:batch_size}
    \end{minipage}
\end{figure}

\subsection{Batch Size}

Batch size is the number of samples to propagate through algorithm at a time. We implemented batch sizes of 32, 64, 128, 256, 512. We observed that the accuracy is unaffected when bit size varies. However, larger batch sizes lead to slower convergence. The result is shown in Figure~\ref{fig:batch_size}. This result is intuitive since batch size represents the number of inputs. As batch size increases, the network receives more data and needs more operations in each layer. Thus, the error accumulates, which drops the converging speed. 

\subsection{Weight Initialization}
The way weights are initialized prior to training also affects the final accuracy of a network \cite{IG}. We thus tested if weight initialization is also sensitive to reduced precision. Because the symmetries of neurons can cause synchronization in the learning process, we used random initial weights instead of uniform weights. Perturbation is achieved by adding a small constant to the fixed random weight initialization (Figure~\ref{fig:acc_init_weight}). 

\begin{figure}[h!]
\centering
\begin{subfigure}[b]{0.4\textwidth}
        \includegraphics[width=\textwidth]{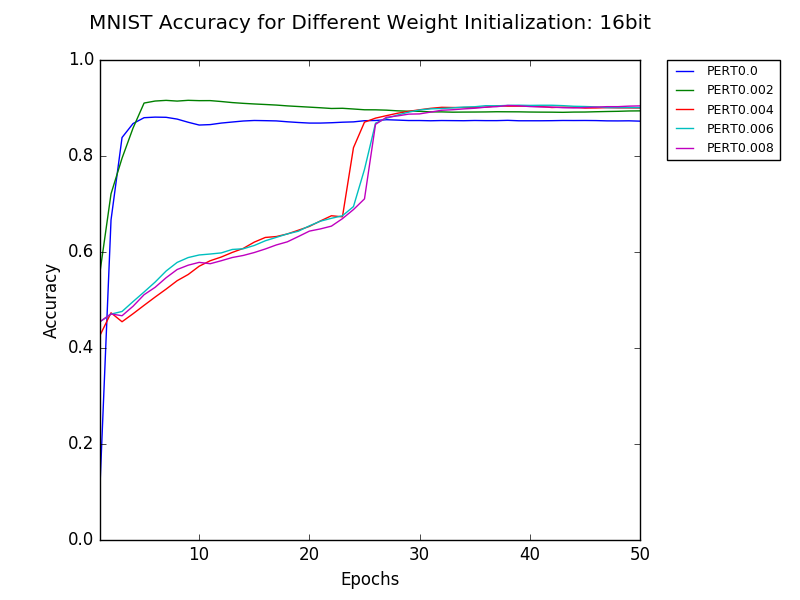}
    \end{subfigure}
    \begin{subfigure}[b]{0.4\textwidth}
        \includegraphics[width=\textwidth]{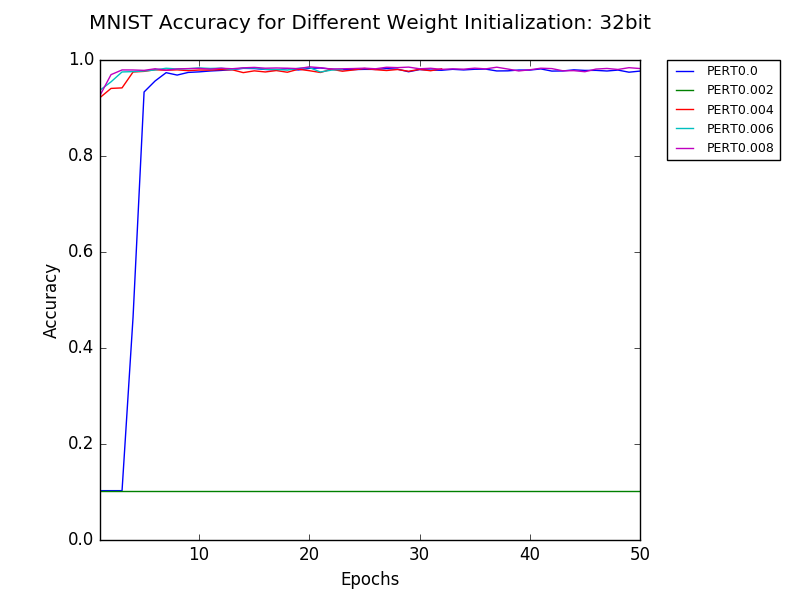}
    \end{subfigure}
    \caption{MNIST Test Accuracy vs. Initial Weights}
    \label{fig:acc_init_weight}
\end{figure}

From Figure~\ref{fig:acc_init_weight}, regardless of the precision, the model converges quickly when perturbation is less than 0.002, while it converges significantly slower when perturbation is larger than 0.004. This result calls up a caution when implementing reduced precision. As reducing precision perturbs the initial weight by truncating the numbers down, the direct impact of weight initialization on neural network accuracy could imply potential harmfulness. 

\subsection{Conclusion}
Reducing the numerical precisions has the following impacts on the parameters we investigated in:
\begin{itemize}
\item In general, the final test accuracy is lower for small bit sizes. 
\item Stochastic rounding converges faster than truncation.
\item Increasing the number of layers tends to affect test accuracy negatively.
\item The number of units in each fully connected layer is independent of test accuracy. 
\item Larger batch sizes take longer to converge, while the accuracy remains the same. 
\end{itemize}

\section{Error Analysis}\label{Ch:analysis}

This section provides a theoretical investigation of how the round-off error accumulates during the forward and backward propagation process in a convolutional neural network. We only present results for the forward propagation. The backpropagation is more complicated and is shown in the Appendix~\ref{sec:backprop}. We use the truncation method described above for the analysis. 

In this analysis, we focus on the convolution and pooling process, while omitting the regularization term, as it is independent of the data. Since pooling does not introduce rounding error, we focus on the convolution. We use a discretized version of convolution, following the definition of \cite{IG}. 

\vspace{-1pt}

\begin{definition}[Discretized Convolution]\label{eq:conv}
Denote $I$ as the inputs to the convolutional layer, $W$ as the weight matrix, and $S$ as the outputs of the layer, then
\[S(i, j) = \sum_{m}\sum_{n}I(i+m, j+n)W(m, n).\]
\end{definition}

Denote $\varepsilon$ as the error of $x$, and $\tilde{x}$ as the approximation of the true value $x$. Let $S_i$ be the output of the $i^{th}$ layer, $W_i$ be the weight matrix of the $i^{th}$ layer, while $M_i$ and $N_i$ denote the height and width of the filter. The result is shown below. 

\begin{proposition}\label{prop:Sn_approx}
Let $M_0=N_0=1$ and $W_0=I$. Given that $W_i\neq0$ for every $i$, 
\begin{align*}
\tilde{S_n}(i, j)\approx S_n(i, j) - \left(\prod_{i=0}^n M_iN_iW_i\right)\cdot \left(\sum_{i=0}^n \frac{1}{W_i}\right)\varepsilon.
\end{align*}

\end{proposition}

\begin{proof}
We use proof by induction. 

When $n=1$, 

\begin{align}\label{eq:conv_approx}
\tilde{S_1}(i, j) &= \sum_{m}\sum_{n}\tilde{I}(i+m, j+n)\tilde{W_1}(m, n)\nonumber \\
&=\sum_{m}\sum_{n}\big(I(i+m, j+n)-\varepsilon\big)\big(W_1(m, n)-\varepsilon\big)\nonumber\allowdisplaybreaks \\
&=\sum_{m}\sum_{n}\big(I(i+m, j+n)W_1(m, n)-(I(i+m, j+n)+W_1(m, n))\varepsilon+\varepsilon^2\big)
\end{align}

Assume that each entry in $I_i$ and $W_i$ are of the same order, we have $\displaystyle \sum_{m}\sum_{n}I(i+m, j+n) = M_1N_1I$ and $\displaystyle \sum_{m}\sum_{n}W_1(m, n) = M_1N_1W_1$. It follows that each entries in $S_1$ has the same order, namely, $S_1(i+m, j+n)\approx M_1N_1IW_1$.

Since $\varepsilon\ll I$ and $\varepsilon\ll W_1$, we expand the sum, omit the second order term, and get 

\begin{align}\label{eq:S1_approx}
\tilde{S_1}(i, j) &= S_1(i,j)-M_1N_1(I+W_1)\varepsilon.
\end{align}

Assume that this equation holds for $n=k-1$. For the case of $n=k$, for simplicity we let $\displaystyle T_k=\left(\prod_{i=0}^{k} M_iN_iW_i\right)\cdot \left(\sum_{i=0}^{k} \frac{1}{W_i}\right)$, and we have
\begin{align*}
\tilde{S_k}(i, j) &= \sum_{m}\sum_{n}\tilde{S}_{k-1}(i+m, j+n)\tilde{W_k}(m, n)\\
&\approx \sum_{m}\sum_{n}\left(S_{k-1}(i+m,j+n) - T_{k-1}\varepsilon\right)\big(W_k(m, n)-\varepsilon\big)\\
&=S_k(i, j)-\sum_{m}\sum_{n}S_{k-1}(i+m,j+n)\varepsilon
 -\sum_m \sum_n T_{k-1} W_k(m, n)\varepsilon + \sum_m \sum_n T_{k-1}\varepsilon^2\\
&\approx S_k(i, j) - \sum_m \sum_n \left(\prod_{i=0}^{k-1} M_iN_iW_i + \left(\prod_{i=0}^{k-1} M_iN_iW_i\right) \left(\sum_{i=0}^{k-1}\frac{1}{W_i}\right)W_k\right)\varepsilon\\
&= S_k(i, j) - \sum_m \sum_n \left(\prod_{i=0}^{k-1} M_iN_i\right) \left(\prod_{i=0}^k W_i \cdot \frac{1}{W_k}\varepsilon + \prod_{i=0}^k W_i \left(\sum_{i=0}^{k-1} \frac{1}{W_i}\right) \varepsilon\right) \\
&\approx S_k(i, j) - \left(\prod_{i=0}^k M_iN_iW_i\right)\cdot \left(\sum_{i=0}^k \frac{1}{W_i}\right)\varepsilon.
\end{align*}
\end{proof}

As shown above, the forward propagation error scales linearly with the dimensions of weight matrices. In terms of layers, the error tends to accumulate even more quickly as the number of layers goes up. Therefore, increasing the number of layers indeed introduces a lot of rounding error, thus drops down the accuracy, as is shown in Section~\ref{sec:layers}.

\section{Future Work}\label{Ch:FutureWork}

\subsection{More Truncation Methods}
Our work has only involved truncation by batch, where the weights are truncated as they are updated after each iteration. A possible next step is to truncate more frequently, which would more closely resemble a hardware restriction where all calculations could only be performed with low precision. Unfortunately, both truncation by layer and truncation by elementary operations are significantly more complicated than truncation by batch. We implemented the truncation by layer method and present some preliminary results in the following section. 

\subsubsection{Truncation by Layer}

We tested the truncation by layer method on the number of dense layers by varying the number of layers while fixing the other parameters. The results are shown in Figure~\ref{fig:dense_layer_by_layer}. 

\begin{figure}[!htb]
    \centering
    \begin{minipage}{.5\textwidth}
        \centering
        \includegraphics[scale=0.4]{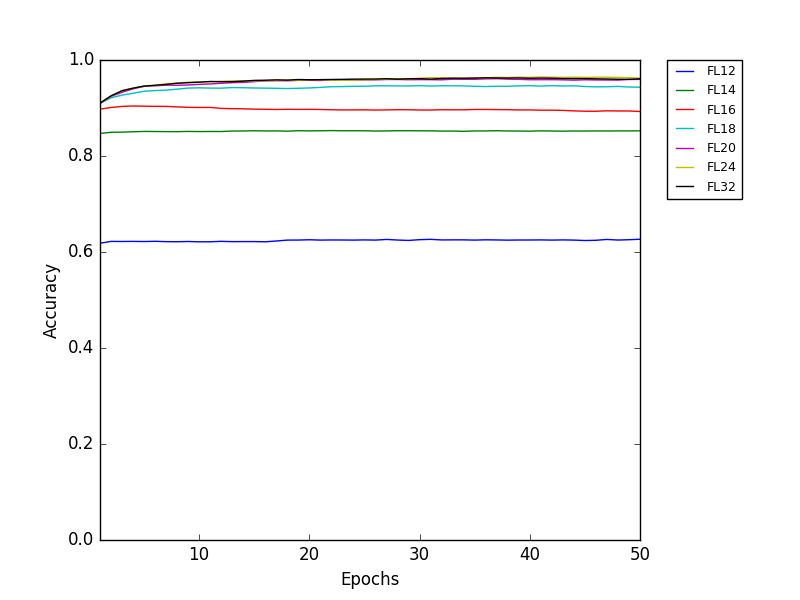}
    \end{minipage}%
    \begin{minipage}{0.5\textwidth}
        \centering
        \includegraphics[scale=0.4]{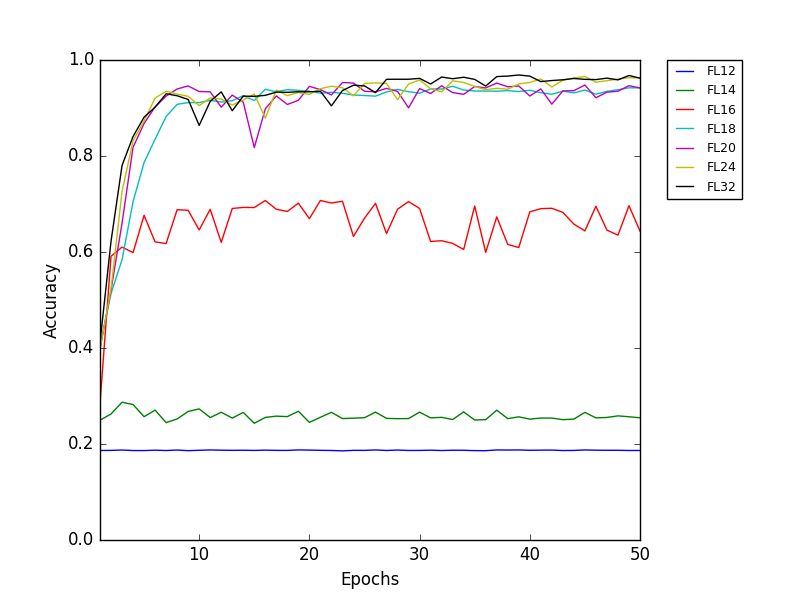}
    \end{minipage}
    \caption{Test Accuracy vs. Dense Layers Using Truncation by Layer}
    \label{fig:dense_layer_by_layer}
\end{figure}

Similar to Section~\ref{sec:layers}, Figure~\ref{fig:dense_layer_by_layer} shows that the test accuracy drops down as the number of layers goes up. In particular, the network with bitsize 16 changes from well-trained to ill-trained when the number of layers goes up from one to four. Also, noticing that in Figure~\ref{fig:acc_dense}, the network with bitsize 16 and four layers is still well-trained using the truncation by batch method, truncation by layers has a higher bitsize requirement to train the network well. This further motivates us to implement truncation by basic operation. 

\subsubsection{Truncation by Basic Operation}
\hspace{16pt} To closely represent hardware limitations where no computations could be done in higher precision, it would be useful to implement truncation after each basic arithmetic operation. However, from our experience, this is impossible in Theano, since each function performs many basic operations. We would have to implement our neural network without using any machine learning libraries, which time would not permit in our case. However, the results would be very interesting to see.

Currently, we have two possible ideas to solve this problem. The first one was to incorporate the SymPy package into Theano functions. SymPy is a package for representing mathematical equations symbolically. If successful, SymPy would have allowed rounding to be performed when evaluating symbolic graphs. However, all operations must be revised to use the SymPy package instead of NumPy which Theano uses conventionally. We were not able to get Theano and SymPy to work in conjunction.

Our second idea was to override the gradient function using finite differences. Although the finite differences method can take numerical data as input, it is considerably more computationally expensive than Theano's gradient function. Moreover, using finite differences would introduce truncation error in addition to the rounding error. Since we would be dealing with various levels of numerical precision, determining the appropriate value of $\epsilon$ to use for each precision will be very time-consuming. On the other hand, it is a potential future direction to implementing more strict truncation methods while using Theano.

\subsection{More Variations to Explore}
\hspace{16pt} Most of our conclusions are based on MNIST, which is relatively simple. Furthermore, our neural network has only two convolution layers and one dense layer. In the future, data sets such as SVHN and CIFAR-100 should be explored to further validate our results. As of frameworks and architectures, we should explore many state-of-art architectures such as LeNet, GoogLeNet, or VGG16 to generalize our results. Also, we have only obtained test error accuracy but have not looked into training errors. We would like to see the differences of training errors on the parameters which we investigated, which shows if our results are affected by overfitting. We could also experiment with different algorithms since we are now only running own data on RMSprop. A few options include Adam, stochastic gradient descent, and momentum. Since we initialize our weights with a random seed, we could also test other seeds to observe if the conclusions still hold. Last but not least, as we have only investigated the error analysis when considering truncation by elementary operation, we plan on to investigate truncation by batch, truncation by layer and stochastic rounding to further validate our results. We also plan to conduct error analysis using rounding methods or stochastic rounding methods, which are more commonly implemented in today's neural network architectures. 

\section*{Acknowledgement}

This research was carried out as part of the 2017 RIPS program at IPAM, the University of California, Los Angeles, and was supported by NSF grant DMS-0931852. We would like to thank Hangjie Ji, Nicholas Malaya, Allen Rush for their mentorship, support, and valuable advice. We would also like to thank Dimi Mavalski and Susana Serna for their help on organizing the RIPS program. We thank AMD Company for their sponsorship and support.

\bibliographystyle{siam}     

\renewcommand\bibname{Selected Bibliography Including Cited Works}
\nocite{*}
\bibliography{Biblio}

\newpage

\appendix

\section{Back Propagation Error Analysis}\label{sec:backprop}

\hspace{16pt} During the training process, back propagation uses the results from forward propagation to update the weight and bias variables of the network. A commonly used method is gradient descent. During the update process, we first compute the gradients of the cost function with respect to the weight and bias variables in each layer. The new weight variables will be obtained by subtracting a product of the gradient and learning rate (a preset constant) from the original values. Computing the gradients requires the chain rule, which complicates this analysis comparing to that for forward propagation.\\\\
Any regularization terms are again avoided and the squared error measure is used for simplicity. Denote $y$ as the output, $y_0$ as true value of test data, $W^{(k)}$ as the weight matrix of the $k^{th}$ layer, $b^{(k)}$ as the bias vector of $k^{th}$ layer, $a^{(k)}$ as the output (after activation) of the $k^{th}$ layer, and $z^{(k)}$ as the output (before activation) of the $k^{th}$ layer. We then have $z^{(k)} = W^{(k)}a^{(k-1)}+b^{(k)}$) and $J$ as the cost function. We follow the algorithms on \cite{IG} and \cite{Git} described as follows:


\begin{enumerate}
\item Compute $$g \leftarrow \nabla_yJ = \nabla_y(\frac{1}{2}(y_0-y)^\top (y_0-y)) = y_0 - y $$
\item Compute $$g \leftarrow \nabla_{z^{(n)}}J = g \cdot \nabla_{z^{(n)}}y = (y_0-y)\odot(y-y^2)$$ (Here $\odot$ is element-wise multiply)
\item Compute 
\begin{align*}
\nabla_{b^{(n)}}J &= g\\
\nabla_{W^{(n)}}J &= g \otimes a^{(n-1)\top}
\end{align*}
(Here $\otimes$ is the outer product)
\item Compute $$g \leftarrow \nabla_{a^{(n-1)}}J = W^{(n)\top} \cdot g$$
\item Repeat Steps 2 to 4 for $n-1$ and so on.
\end{enumerate}

Assuming that the test data are correctly stored, for step 1 we have 
\begin{align}
\tilde{g} &= y_0 - \tilde{y} = g - \varepsilon_y
\end{align}

For step 2 we have 
\begin{align}
\tilde{\nabla}_{z^{(n)}}J &= (g-\varepsilon_y)\cdot (y-\varepsilon_y - (y - \varepsilon_y)^2)\\
&= (g-\varepsilon_y)\cdot (y-\varepsilon_y)(1-y+\varepsilon_y)\\
&= \nabla_{z^{(n)}}J + (-g+2gy-y+y^2)\varepsilon_y + (-g+1-2y)\varepsilon_y^2 + \varepsilon_y^3\label{eq:grad_approx}
\end{align}

From equation~\eqref{eq:grad_approx}, we can see that the error may or may not scale linearly depending on the values of $g$ and $y$. Therefore, predicting the dependencies on previous layers is much more complicated, and so is the approximation for the scaling of back propagation error.

In conclusion, since forward propagation error is guaranteed to have a linear scaling, forward propagation process is dominating the sensitivity of reduced precision compared to back-propagation when truncation is applied. A potential reason why back-propagation is usually observed to cause the error might be that after forward propagation, the accumulation error is already close to the breaking threshold. Adding an extra back propagation error causes the accuracy to fall apart completely.

\end{document}